\documentclass[twoside,11pt]{article} % For pdflatex
\pdfoutput=1

\usepackage[preprint]{jmlr2e}
\usepackage{amsmath}
\usepackage{mathtools} 
\usepackage{algorithm,algorithmic}
\usepackage{comment}
\usepackage{caption}
\usepackage[section]{placeins}

\newcommand{\qed}{\hfill\BlackBox\\[2mm]}

\ShortHeadings{Reconsidering Dependency Networks}{Takabatake and Akaho}
\firstpageno{1}

\begin{document}

\title{Reconsidering Dependency Networks from an Information Geometry Perspective}

\author{\name Kazuya Takabatake \email takabatake@roy.hi-ho.ne.jp\\
	\addr Independent Researcher\\
	Izumisano, Osaka, Japan\\
	\AND
	\name Shotaro Akaho \email akaho@ism.ac.jp\\
	\addr The Institute of Statistical Mathematics\\
	Tachikawa, Tokyo, 190-8565, Japan}
\editor{}

\maketitle

\begin{abstract}
	Dependency networks \citep{heckerman2000dependency} provide a flexible
	framework for modeling complex systems with many variables by combining
	independently learned local conditional distributions through
	pseudo-Gibbs sampling.
	Despite their computational advantages over Bayesian and Markov
	networks, the theoretical foundations of dependency networks remain
	incomplete, primarily because their model distributions---defined as
	stationary distributions of pseudo-Gibbs sampling---lack closed-form
	expressions.
	This paper develops an information-geometric analysis of pseudo-Gibbs
	sampling, interpreting each sampling step as an m-projection onto a
	full conditional manifold.
	Building on this interpretation, we introduce the full conditional
	divergence and derive an upper bound that characterizes the location of
	the stationary distribution in the space of probability distributions.
	We then reformulate both structure and parameter learning as
	optimization problems that decompose into independent subproblems for
	each node, and prove that the learned model distribution converges to
	the true underlying distribution as the number of training samples
	grows to infinity.
	Experiments confirm that the proposed upper bound is tight in practice.
\end{abstract}

\begin{keywords}
	Graphical models, Dependency networks, Pseudo-Gibbs sampling,
	Information geometry, m-projection, Full conditional divergence,
	Pseudo-likelihood
\end{keywords}

\section{Introduction}
The primary purpose of this paper is to investigate dependency networks
\citep{heckerman2000dependency} from an information-geometric perspective
and to provide new theoretical insights into their behavior.

For a large collection of random variables $X=\{X_i\}_{i=0}^{n-1}$,
learning the joint distribution $p^*(X)$ from training data requires
substantial computational resources, where $p^*$ denotes the true
underlying distribution.
Bayesian networks and Markov networks \citep{koller2009probabilistic}
are well-established tools for modeling such joint distributions;
however, determining their graphical structure becomes computationally
prohibitive as $n$ grows large.

In contrast, learning a single \emph{full conditional distribution}
\citep{gilks1995markov} $p^*(X_i\mid X_{-i})$, where $X_{-i}$ denotes
all variables except $X_i$, is far more tractable because it reduces to
a univariate regression problem for which efficient algorithms are
available \citep{breiman2017classification}.
If a joint distribution can be reconstructed from a collection of
estimated full conditional distributions
$\{\hat{p}_i(X_i\mid X_{-i})\}_{i=0}^{n-1}$, one obtains a learning
framework that models the joint distribution at reduced computational
cost.
This idea dates back to \citet{besag1974spatial,besag1975statistical},
who studied Markov random fields and introduced the notion of
pseudo-likelihood.
\citet{heckerman2000dependency} later coined the term ``dependency
network'' and formalized the framework with a directed graph structure.

A dependency network is a graphical model built on this strategy: each
node employs a regression model to estimate its own full conditional
distribution.
After training, a dependency network generates samples via a Markov
chain Monte Carlo (MCMC)~\citep{gilks1995markov} procedure called
\emph{pseudo-Gibbs sampling} \citep{heckerman2000dependency}.

The model distribution of a Bayesian network is expressed as
$\pi(X)=\prod_i\pi(X_i\mid Y_i)$, and that of a Markov network as
$\pi(X)=Z^{-1}\prod_c\phi_c(X_c)$
\citep{koller2009probabilistic}.
Because these models admit closed-form expressions for~$\pi$, maximum
likelihood estimation is directly applicable.
In contrast, the model distribution of a dependency network is the
stationary distribution of a Markov chain and has no closed-form
expression; consequently, maximum likelihood estimation is not
applicable.
This absence of a closed-form expression has been the main obstacle to
the theoretical analysis of dependency networks.

\citet{heckerman2000dependency} analyzed pseudo-Gibbs sampling as a
perturbation of standard Gibbs sampling
\citep{geman1984stochastic}.
However, their perturbation argument is insufficient to characterize
general dependency networks, since even a small change in a transition
matrix can substantially alter its stationary distribution.
\citet{Takabatake12,takabatake2021reconsidering} interpreted pseudo-Gibbs sampling as iterative m-projections \citep{Amari95} onto \emph{full conditional manifolds},
but these works remained unpublished as refereed papers.
\citet{kuo2019pseudo} independently established the same geometric
interpretation.

This paper builds on
\citet{Takabatake12,takabatake2021reconsidering} and makes the
following contributions:
\begin{enumerate}
	\item We generalize the concept of \emph{information source}
	(Section~\ref{sec:node}) and establish its connection to
	Shannon entropy minimization.
	\item We introduce the \emph{full conditional divergence}, a new
	pseudo-distance closely related to Besag's pseudo-likelihood,
	and derive an upper bound (the FC-limit) that characterizes
	the location of the stationary distribution in the space of
	probability distributions (Section~\ref{sec:information_geometry_of_pseudo-gibbs_sampling}).
	\item We reformulate structure and parameter learning as
	optimization problems that decompose into independent
	subproblems for each node, and prove that the learned model
	distribution converges to the true underlying distribution as
	the training sample size grows to infinity
	(Section~\ref{sec:learning}).
	\item We provide a theoretical analysis of inference by conditional
	pseudo-Gibbs sampling
	(Section~\ref{sec:inference_by_pseudo-gibbs_sampling}).
\end{enumerate}
Compared with Bayesian and Markov networks, dependency networks offer
significant computational advantages because the learning task
decomposes into independent subproblems, one for each node.
The information-geometric framework developed here provides the
theoretical foundation needed to exploit this advantage rigorously.

The remainder of this paper is organized as follows.
Section~\ref{sec:dependency_network} reviews dependency networks.
Section~\ref{sec:information_geometry_of_pseudo-gibbs_sampling}
develops an information-geometric analysis of pseudo-Gibbs sampling.
Section~\ref{sec:learning} presents structure and parameter learning
algorithms with convergence guarantees.
Section~\ref{sec:inference_by_pseudo-gibbs_sampling} discusses
probabilistic inference.
Section~\ref{sec:related_work} reviews related work, and
Section~\ref{sec:conclusion} concludes the paper.
Proofs of all theorems are collected in the Appendix.

\section{Dependency Network}\label{sec:dependency_network}
We begin by establishing the notation used throughout this paper.
Random variables are denoted by capital letters (e.g., $X$), and their
realized values by the corresponding lowercase letters (e.g., $x$).
The set of values that $X$ can take is also denoted by $X$; the
intended meaning will be clear from context.
Throughout this paper, we consider only finite discrete
variables.\footnote{Some of the results may extend to continuous
	variables, but we restrict attention to the finite discrete case to
	keep the mathematical treatment self-contained.}
For any set $A$, $|A|$ denotes its cardinality; in particular, $|X|$
denotes the number of possible values of $X$.
The number of variables in the network is denoted by $n$.
For a subset $J\subseteq[0,n)$, $X_J$ denotes the collection
$\{X_i\}_{i\in J}$, and $X_{-i}$ denotes all variables except $X_i$.
For brevity, a tuple of variables $(X,Y,Z)$ is written as $XYZ$.
We write $N$ for the number of training samples and $N^o$ for the
number of samples in the pseudo-Gibbs sampling output.
The expectation of a function $f(X)$ under a distribution $p(X)$ is
denoted by
\[\langle f(X) \rangle_{p(X)}=\sum_{x \in X} p(x) f(x).\]
When $X$ is clear from context, we write simply
$\langle f(X)\rangle_p$.
The set of all probability distributions over $X$ is denoted by
$\mathcal P(X)$, or simply $\mathcal P$ when $X$ is understood.
We write $p^*$ for the true underlying distribution, $p^D$ for the
empirical distribution of the training data, and $\pi$ for the model
distribution of a dependency network.
The entropy of $p(X)$ is defined as
$H(p(X))=-\langle \log p(X) \rangle_p$, and the conditional entropy of
$p(X\mid Y)$ as
$H(p(X\mid Y)) =-\langle\log p(X\mid Y) \rangle_{p(XY)}$.
The word ``ergodic'' means irreducible and aperiodic.
Throughout this paper, $\log$ denotes the natural logarithm.
In all figures, solid lines represent m-flat manifolds and dotted lines
represent e-flat manifolds.

\subsection{Node}\label{sec:node}
\begin{figure}[t]
	\centering
	\includegraphics{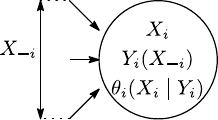}
	\caption{Node $i$.}\label{fig:node}
\end{figure}

A dependency network consists of $n$ nodes.
Figure~\ref{fig:node} illustrates the $i$-th node, which comprises
three elements: a variable $X_i$, a deterministic function
$Y_i(X_{-i})$ called the \emph{information source}, and a conditional
distribution table $\theta_i(X_i\mid Y_i)$.

The information source $Y_i$ summarizes the information in $X_{-i}$
that is relevant to the distribution of $X_i$.
In many practical cases $|Y_i|\ll|X_{-i}|$, meaning that $Y_i$
depends on only a small subset of the variables in $X_{-i}$.
Although $Y_i$ is defined as a function of $X_{-i}$, it can also be
viewed as a function of $X$ that ignores $X_i$; we write
$Y_i(X_{-i})$ or $Y_i(X)$ depending on the context.

The conditional distribution table $\theta_i$ specifies a distribution
over $X_i$ given $Y_i$, and is used in \emph{pseudo-Gibbs sampling}
(Section~\ref{sec:pseudo-gibbs_sampling}).
Since $Y_i$ is itself a function of $X_{-i}$, the notations
$\theta_i(X_i\mid Y_i)$ and $\theta_i(X_i\mid X_{-i})$ are used
interchangeably.

\subsection{Pseudo-Gibbs Sampling}\label{sec:pseudo-gibbs_sampling}
Pseudo-Gibbs sampling \citep{heckerman2000dependency} is an MCMC method
that constructs a joint distribution from a set of conditional
distribution tables.
Algorithm~\ref{alg:pseudo} presents the procedure.

\begin{algorithm}[t]
	\caption{Pseudo-Gibbs sampling}\label{alg:pseudo}
	$X=\{X_i\}_{i=0}^{n-1}$: Variables in the network\\
	$\{X^t\}_{t=0}^{N^o-1}$: MCMC output
	\begin{algorithmic}[1]
		\STATE $X\gets x^0$ (arbitrary initial value)
		\FOR{$t=0$ \TO $N^o-1$}
		\STATE $X^t\gets X$
		\STATE Select a node $i$ \label{line:select_a_node}
		\STATE $y_i \gets Y_i(X)$
		\STATE Draw $x_i\sim\theta_i(X_i\mid y_i)$
		\STATE $X_i\gets x_i$
		\ENDFOR
		\RETURN $\{X^t\}_{t=0}^{N^o-1}$
	\end{algorithmic}
\end{algorithm}

Throughout this paper, we assume that the Markov chain defined by pseudo-Gibbs sampling is ergodic; this guarantees the existence of a unique stationary distribution $\pi$ such that
$\lim_{t\to\infty}p^t(X)=\pi(X)$, where $p^t$ denotes the chain's distribution at time $t$
\citep{gilks1995markov,seneta2006non}.

To distinguish it from pseudo-Gibbs sampling, we refer to the original
Gibbs sampling \citep{geman1984stochastic} as \emph{genuine Gibbs
	sampling}.
Genuine Gibbs sampling is the special case of pseudo-Gibbs sampling in
which there exists a distribution $\pi(X)$ satisfying
\[
\forall i\in[0,n),\quad
\pi(X_i\mid X_{-i}) \equiv \theta_i(X_i\mid X_{-i}).
\]
Under ergodicity, the distribution of $X^t$ converges to $\pi(X)$ as
$t \to \infty$ in genuine Gibbs sampling \citep{gilks1995markov}.

There are two common strategies for selecting a node at
line~\ref{line:select_a_node}.
The first, called \emph{sequential scan} \citep{he2016scan}, cycles
through nodes in a fixed order $0,1,\dots,n-1,0,1,\dots$.
The second, called \emph{random scan} \citep{he2016scan}, selects node
$i$ with probability $c_i$, where
\begin{equation}\label{eq:c_i}
	\sum_{i=0}^{n-1}c_i=1.
\end{equation}
Unless stated otherwise, we set $c_i=1/n$.

Assuming ergodicity of pseudo-Gibbs sampling with random scan, the
following properties hold \citep{gilks1995markov}:

\begin{description}
	\item[Existence of stationary distribution:]
	There exists a unique \emph{stationary distribution}
	$\pi(X)$ such that for any initial value $x^0$,
	\[
	\pi(X) = \lim_{t \to \infty} p(X^t\mid x^0).
	\]
	\item[Monte Carlo integration:]
	For any function $f(X)$ with
	$\langle|f(X)|\rangle_\pi<\infty$,
	\begin{equation}\label{eq:ergodicity}
		\lim_{N^o \to \infty}
		\frac1{N^o}\sum_{t=0}^{N^o-1}f(X^t)
		= \langle f(X) \rangle_\pi
		\quad \text{(a.s.)}.
	\end{equation}
\end{description}

Substituting the indicator function $f(X) =\mathbf{1}(X=x)$ into
Eq.~\eqref{eq:ergodicity} gives
\[
\lim_{N^o \to \infty} \frac{N^o_x}{N^o} = \pi(x)
\quad \text{(a.s.)},
\]
where $N^o_x=|\{t\mid X^t =x,\, t\in[0,N^o)\}|$ denotes the
number of occurrences of $x$ in the MCMC output.
This shows that $\pi(x)$ can be estimated via Monte Carlo methods even
when it cannot be determined analytically.

Pseudo-Gibbs sampling with sequential scan defines an
\emph{inhomogeneous} Markov chain \citep{seneta2006non} and therefore
does not possess a stationary distribution in general.
However, the subsequence $X^i, X^{i+n}, X^{i+2n}, \dots$ forms a
\emph{homogeneous} Markov chain with stationary distribution
$\pi^i(X)$ under the ergodicity assumption.
Consequently, the limit in Eq.~\eqref{eq:ergodicity} still exists, and

\begin{equation}\label{eq:ordered}
	\lim_{N^o\to\infty}\frac{N^o_x}{N^o}
	= \frac{\pi^0(x) + \dots + \pi^{n-1}(x)}{n}.
\end{equation}

Based on these observations, we state the following working hypothesis,
which is central to dependency network learning and will be justified
in Section~\ref{sec:information_geometry_of_pseudo-gibbs_sampling}.

\begin{quote}
	\textbf{Hypothesis.}\quad
	If each conditional distribution table
	$\theta_i(X_i\mid X_{-i})$ provides a good estimate of
	$p^*(X_i\mid X_{-i})$, then the stationary distribution of
	pseudo-Gibbs sampling provides a good estimate of $p^*(X)$.
\end{quote}

\section{Information Geometry of Pseudo-Gibbs Sampling}
\label{sec:information_geometry_of_pseudo-gibbs_sampling}
In this section, we examine pseudo-Gibbs sampling from the perspective
of information geometry.
This analysis leads to an interpretation of each sampling step as an
m-projection \citep{Amari95} onto a certain manifold, and motivates the
problem of making the stationary distribution close to a desired target
distribution.

We begin with definitions and standard results from information
geometry.
In information geometry, a distribution is represented as a point in
the space of distributions $\mathcal{P}$. For two distinct distributions $p_0, p_1 \in \mathcal{P}$, the
one-dimensional manifold
\[
\{p_\lambda \in \mathcal{P} \mid p_\lambda(X) = (1-\lambda)p_0(X)
+ \lambda p_1(X),\ \lambda \in \mathbb{R}\}
\]
is called the \emph{m-geodesic} $p_0p_1$, and the
one-dimensional manifold
\[
\{p_\lambda \in \mathcal{P} \mid \log p_\lambda(X) = (1-\lambda)
\log p_0(X) + \lambda \log p_1(X) - \log Z\}
\]
is called the \emph{e-geodesic} $p_0p_1$, where $Z$
is a normalizing constant
\citep{Amari00,amari2016information}.

A manifold $\mathcal{M}$ is said to be \emph{m-flat} if the m-geodesic
$p_0p_1$ is contained in $\mathcal{M}$ for every pair of distinct
distributions $p_0, p_1 \in \mathcal{M}$.
Similarly, $\mathcal{M}$ is said to be \emph{e-flat} if the e-geodesic
$p_0p_1$ is contained in $\mathcal{M}$ for every such pair.

For distributions $p$ and $q$, the Kullback--Leibler (KL) divergence is
defined by
\[
KL(p\|q) = \left\langle \log\frac{p(X)}{q(X)}
\right\rangle_p.
\]

An m-geodesic $pq$ and an e-geodesic $qr$ are said to be
\emph{orthogonal} if and only if
\[
\sum_{x \in X} (p(x) - q(x))(\log q(x) - \log r(x)) = 0.
\]
When this condition holds, the \emph{Pythagorean theorem}
\begin{align*}
	KL(p\|r)
	& = \left\langle\log\frac{p}{q}\right\rangle_p
	+ \left\langle\log\frac{q}{r}\right\rangle_p          \\
	& = \left\langle\log\frac{p}{q}\right\rangle_p
	+ \left\langle\log\frac{q}{r}\right\rangle_q
	= KL(p\|q)+KL(q\|r)
\end{align*}
follows, where the second equality uses the orthogonality condition.

We define the \emph{m-projection} onto a manifold $\mathcal{M}$ as the
operator $\Pi_m(\mathcal{M})$ that maps $p$ to
\[
p\,\Pi_m(\mathcal{M})=\arg\min_{q\in\mathcal{M}}
KL(p\|q).
\]
As illustrated in Figure~\ref{fig:m-projection}, when $\mathcal{M}$ is
e-flat, $p\Pi_m(\mathcal{M})$ is uniquely determined by the Pythagorean
theorem.

\begin{figure}[t]
	\centering
	\includegraphics{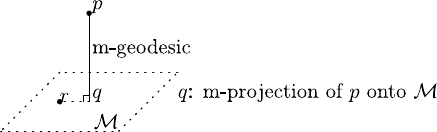}
	\caption{m-projection onto an e-flat manifold.}
	\label{fig:m-projection}
\end{figure}

With these tools in place, we turn to the information geometry of
Gibbs sampling.
We define the \emph{full conditional manifold} of node $i$ as
\begin{equation}\label{eq:E}
	E(\theta_i)=\{p \in \mathcal{P} \mid
	p(X_i \mid X_{-i})=\theta_i(X_i \mid X_{-i})\}.
\end{equation}
Any distribution $p \in \mathcal{P}$ can be written as
$p(X_{-i})\,p(X_i \mid X_{-i})$, and the manifold $E(\theta_i)$
constrains the full conditional component of this factorization.

\begin{theorem}\label{thm:e-flat}
	Every full conditional manifold is both e-flat and m-flat.
\end{theorem}
\begin{proof}
	See Appendix.
\end{proof}

Let $\{X^t\}_{t=0}^\infty$ be a Markov chain generated by pseudo-Gibbs
sampling, and let $p^t=p(X^t)\in\mathcal{P}$ denote the distribution at
time $t$.
Suppose that node $i$ is updated at time $t$.
Then $p^t$ transitions to
\[
p^{t+1}(X) = p^t(X_{-i})\,\theta_i(X_i \mid X_{-i}),
\]
that is, the full conditional distribution $p^t(X_i \mid X_{-i})$ is
replaced by $\theta_i(X_i \mid X_{-i})$.
The following theorem shows that this transition is an m-projection.

\begin{theorem}\label{thm:m-projection}
	\[
	p\Pi_m(E(\theta_i))=p(X_{-i})\theta_i(X_i \mid X_{-i}).
	\]
	That is, updating node $i$ moves the distribution $p$ to its
	m-projection onto $E(\theta_i)$.
\end{theorem}
\begin{proof}
	See Appendix.
\end{proof}

A notable property of this m-projection is its linearity:
\[
\bigl((1-\lambda)p_0 + \lambda p_1\bigr)\Pi_m(E(\theta_i))
= (1-\lambda)p_0\Pi_m(E(\theta_i))
+ \lambda p_1\Pi_m(E(\theta_i)),
\]
which follows directly from Theorem~\ref{thm:m-projection}.

\subsection{Pseudo-Gibbs Sampling with Sequential Scan}
Pseudo-Gibbs sampling with sequential scan defines an inhomogeneous
Markov chain whose transition matrix changes with $t$ in
Algorithm~\ref{alg:pseudo}, which complicates its theoretical analysis.
Nevertheless, it admits a natural geometric interpretation.

\begin{figure}[t]
	\centering
	\includegraphics{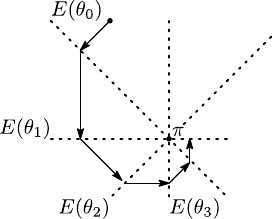}
	\caption{Movement of $p^t$ in genuine Gibbs sampling with
		sequential scan.}
	\label{fig:genuine_gs}
\end{figure}
\begin{figure}[t]
	\centering
	\includegraphics{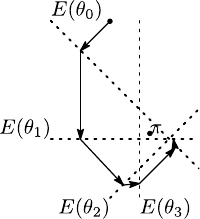}
	\caption{Movement of $p^t$ in pseudo-Gibbs sampling with
		sequential scan.}
	\label{fig:pseudo-gs}
\end{figure}

Figure~\ref{fig:genuine_gs} illustrates the movement of
$p^t\,({=}\,p(X^t))$ in genuine Gibbs sampling.
In this case, all full conditional manifolds intersect at a unique
point $\pi$.
At the first step, $X_0$ is updated: $p^0$ is m-projected onto
$E(\theta_0)$.
Next, $X_1$ is updated: $p^1$ is m-projected onto $E(\theta_1)$, and
so on.
This depiction makes the convergence $p^t \to \pi$ geometrically
intuitive.

In contrast, Figure~\ref{fig:pseudo-gs} illustrates pseudo-Gibbs
sampling, where the full conditional manifolds do not share a common
intersection.
If every manifold passes close to a certain point $\pi$, then $p^t$
remains confined to a small neighborhood of $\pi$ after sufficiently
many transitions.
More precisely, $p^t$ moves along the cyclic orbit
$\pi^0\to\dots \to \pi^{n-1} \to \pi^0$, where each $\pi^i$
is defined in Eq.~\eqref{eq:ordered}, and the centroid of
$\{\pi^i\}$ becomes the model distribution.

We note that Figures~\ref{fig:genuine_gs}
and~\ref{fig:pseudo-gs} provide geometric intuition only and do not
constitute a rigorous proof of convergence.

\subsection{Pseudo-Gibbs Sampling with Random Scan}
Pseudo-Gibbs sampling with random scan defines a homogeneous Markov
chain, making its theoretical analysis more tractable than that of
sequential scan.

\begin{figure}[t]
	\centering\includegraphics{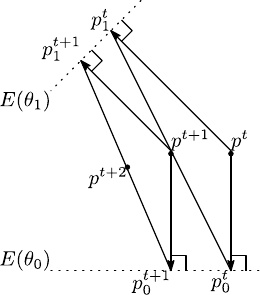}
	\caption{Movement of $p^t$ in pseudo-Gibbs sampling with random
		scan.}\label{fig:movement}
\end{figure}

Figure~\ref{fig:movement} illustrates the movement of $p^t$ for the
case $n=2$ and $c_i=1/2$.
Here $p^t$ denotes the distribution of the Markov chain at time $t$,
and $p^t_i$ denotes the m-projection of $p^t$ onto
$E(\theta_i)$.
When the identity of the updated node is not observed,
\begin{equation}\label{eq:p^(t+1)}
	p^{t+1}=\sum_i c_i p^t_i,
\end{equation}
where the $c_i$ are the selection probabilities in
Eq.~\eqref{eq:c_i}.
By the ergodicity assumption, $\lim_{t\to\infty}p^t=\pi$ holds
regardless of the initial distribution $p^0$.

Pseudo-Gibbs sampling can be viewed as a strategic game
\citep{osborne1994course} played by $n+1$ players, in which each
player minimizes its own cost:
\[
\begin{cases}
	p^{t+1}_i=\arg\min_{q\in E(\theta_i)}KL(p^t\|q)    & \text{for } p_i, \\
	p^{t+1}=\sum_i c_i p^t_i
	=\arg\min_{q\in\mathcal{P}}\sum_i c_i KL(p^t_i\|q) & \text{for } p
\end{cases}.
\]
For the stationary distribution $\pi$, the tuple
$(\pi,\{\pi_i\})$ constitutes a \emph{Nash equilibrium}
\citep{osborne1994course}, and convergence to this equilibrium is
guaranteed by ergodicity.

We extend the notion of KL-divergence to the divergence between a
distribution $p$ and a manifold $\mathcal{M}$:
\begin{align*}
	KL(p\|\mathcal{M}) & = \min_{q\in\mathcal{M}}KL(p\|q) \\
	& = KL(p\|p\Pi_m(\mathcal{M})).
\end{align*}
For $\mathcal{M}=E(\theta_i)$, this reduces to
\[
KL(p\|E(\theta_i))=\left\langle\log\frac{p(X_i\mid X_{-i})}{\theta_i(X_i\mid X_{-i})}\right\rangle_p.
\]

We now introduce a central quantity of this paper.
The \emph{full conditional divergence} is defined by
\begin{equation}\label{eq:FC}
	FC(p\|q) = \sum_i c_i\, KL(p\|E_i(q)),\quad
	E_i(q)=\{r\in\mathcal{P}\mid r(X_i\mid X_{-i})=q(X_i\mid X_{-i})\},
\end{equation}
which is the weighted average of KL-divergences from $p$ to the
full conditional manifolds $E_i(q)$
(see Figure~\ref{fig:fcd} for a geometric illustration).
If $p$ and $q$ satisfy
$\forall x\,(p(x)>0\Rightarrow q(x)>0)$, then
$FC(p\|q)<\infty$.

\begin{figure}[t]
	\centering\includegraphics{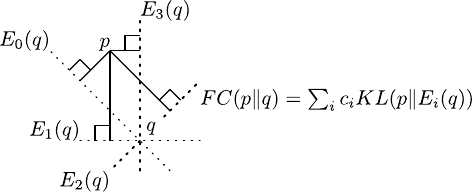}
	\caption{Geometric interpretation of
		$FC(p\|q)$.}\label{fig:fcd}
\end{figure}

The relationship between full conditional divergence and Besag's
\emph{pseudo-likelihood}
\citep{besag1975statistical,koller2009probabilistic} parallels the
relationship between KL-divergence and likelihood.
Rewriting Eq.~\eqref{eq:FC} gives
\begin{equation}\label{eq:FC_rewrite}
	FC(p\|q)=\left\langle\sum_i c_i\log p(X_i\mid X_{-i})\right\rangle_p
	-\underbrace{\left\langle\sum_i c_i\log q(X_i\mid X_{-i})\right\rangle_p}
	_{\text{pseudo-log-likelihood}},
\end{equation}
while KL-divergence takes the form
\begin{equation}\label{eq:KL}
	KL(p\|q)=\left\langle\log p(X)\right\rangle_p
	-\underbrace{\left\langle\log q(X)\right\rangle_p}_{\text{log-likelihood}}.
\end{equation}
In both equations, the first term cancels the second when $p=q$.
Pseudo-likelihood was introduced by \citet{besag1975statistical} as a
computationally convenient surrogate for likelihood that avoids the
expensive computation of the partition function.
Just as minimizing KL-divergence is equivalent to maximizing
likelihood, minimizing full conditional divergence is equivalent to
maximizing pseudo-likelihood.

Furthermore, both KL-divergence and full conditional divergence are
instances of Bregman divergence \citep{banerjee2005clustering}.
Let $\Lambda$ be a convex subset of $\mathbb{R}^k$.
For a strictly convex differentiable function
$f\colon\Lambda\to\mathbb{R}$, called the \emph{Bregman function}, the
Bregman divergence is defined as
\[
B_f(p\|q)=f(p)-f(q)-\nabla f(q)\cdot(p-q),
\]
where $\nabla f(q)$ denotes the gradient of $f$ at $q$ and $\cdot$
denotes the Euclidean inner product.
We treat a distribution $p\in\mathcal{P}$ as a vector in
$\mathbb{R}^{|X|}$ whose $x$-th component is $p(x)$.
KL-divergence is the Bregman divergence induced by the negative
entropy,
\[
f(p)=-H(p(X))=\left\langle\log p(X)\right\rangle_p,
\]
and full conditional divergence is the Bregman divergence induced by
the weighted sum of negative conditional entropies,
\[
f(p)=-\sum_i c_i H(p(X_i\mid X_{-i}))
=\sum_i c_i\left\langle\log p(X_i\mid X_{-i})\right\rangle_p.
\]

A notable feature of full conditional divergence is that it captures
the dependency structure among variables, whereas KL-divergence treats
all variables as a single concatenated variable and therefore does not
reflect any such structure.

The following inequality relates the two divergences (derivation in
Appendix):
\begin{equation}\label{eq:FC<=KL}
	FC(p\|q)\le KL(p\|q).
\end{equation}
An immediate consequence is that
$KL(p\|q)<\infty$ implies $FC(p\|q)<\infty$.

The next theorem, which is the main result of this section,
characterizes the location of the stationary distribution.

\begin{theorem}\label{thm:FC}
	Let $p\in\mathcal{P}$ be an arbitrary distribution and let $\pi$
	be the stationary distribution of pseudo-Gibbs sampling with
	random scan using the conditional distribution tables
	$\{\theta_i\}$.
	Then,
	\begin{equation}\label{eq:FC-limit}
		FC(p\|\pi)\le\sum_i c_i KL(p\|E(\theta_i)).
	\end{equation}
\end{theorem}
\begin{proof}
	See Appendix.
\end{proof}

This theorem justifies the hypothesis stated in
Section~\ref{sec:dependency_network}: if each conditional distribution
table closely approximates the true full conditional, the right-hand
side of Eq.~\eqref{eq:FC-limit} is small, which in turn confines
$\pi$ to a small neighborhood of $p$ in the full conditional
divergence sense.

We refer to the right-hand side of Eq.~\eqref{eq:FC-limit} as the
\emph{FC-limit} of $p$ and denote it by $FC_{\lim}(p)$.
To evaluate the tightness of this bound, we experimentally compare
$FC(p^D\|\pi)$ and $FC_{\lim}(p^D)$ on the following four datasets:

\begin{description}
	\item[Ising4x3S/L:] Data drawn from a $4\times3$ Ising model
	with binary variables.
	Ising4x3S contains 1{,}000 samples and Ising4x3L contains
	100{,}000 samples.
	\item[RB12-21S/L:] Data drawn from a Bayesian network with 12
	nodes and 21 edges, with binary variables.
	RB12-21S contains 1{,}000 samples and RB12-21L contains
	100{,}000 samples.
\end{description}

To compute $\pi$ exactly, we restrict the state space to
$|X|=2^{12}$.\footnote{The stationary distribution $\pi$ is computed
	as the solution of a system of $|X|$ linear equations.}
We construct a dependency network using the learning algorithm
described in Section~\ref{sec:learning} and train it on each dataset.

\begin{table}[t]
	\centering
	\caption{Comparison of $FC(p^D\|\pi)$ and
		$FC_{\lim}(p^D)$.}\label{table:FC-limit}
	\begin{tabular}{lcc}
		\hline
		Training data & $FC(p^D\|\pi)$     & $FC_{\lim}(p^D)$   \\\hline
		Ising4x3S     & $4.0\times10^{-3}$ & $5.3\times10^{-3}$ \\
		Ising4x3L     & $1.1\times10^{-3}$ & $1.1\times10^{-3}$ \\
		RB12-21S      & $1.5\times10^{-1}$ & $1.6\times10^{-1}$ \\
		RB12-21L      & $6.8\times10^{-3}$ & $7.4\times10^{-3}$ \\\hline
	\end{tabular}
	\newline
	\smallskip
	$p^D$: empirical distribution of training data\\
	$\pi$: learned model distribution\\
	Unit: nat ($= 1.44$ bit).
\end{table}

Table~\ref{table:FC-limit} reports the results.
The small gaps between $FC(p^D\|\pi)$ and $FC_{\lim}(p^D)$ confirm
that the FC-limit provides a tight upper bound in practice.

\section{Learning}\label{sec:learning}
As shown in the preceding section, although the model distribution
$\pi$ of a dependency network has no closed-form expression,
Theorem~\ref{thm:FC} guarantees that $\pi$ can be brought close to a
desired target distribution $p$ by reducing the FC-limit
$FC_{\lim}(p)$ in Eq.~\eqref{eq:FC-limit}.
Learning can therefore be performed by minimizing a cost function that
incorporates this bound.

Let $cost(Y_i,\theta_i)$ be such a cost function.
The learning task then decomposes into the following two-stage
minimization:
\[
\min_{Y_i,\theta_i}cost(Y_i,\theta_i)
=\underbrace{\min_{Y_i}\underbrace{\min_{\theta_i}cost(Y_i,\theta_i)}_{\text{parameter learning}}}_{\text{structure learning}}.
\]
Parameter learning fills in the conditional distribution table
$\theta_i(X_i\mid Y_i)$, while structure learning constructs the
information source function $Y_i(X)$.

\subsection{Cost Function}
Let $p^*$ denote the true underlying distribution.
By Theorem~\ref{thm:FC}, a natural choice of cost function is
\[
cost^* = \sum_i c_i\, cost^*_i, \qquad
cost^*_i = KL(p^*\|E(\theta_i)).
\]
We refer to $cost^*$ as the \emph{ideal cost}.
Because $p^*$ is unknown, the ideal cost cannot be computed directly.
We therefore replace $p^*$ by the empirical distribution $p^D$ and add
a non-negative complexity penalty $R(k_i,N)$ to guard against
overfitting, where $k_i=|Y_i|(|X_i|-1)$ denotes the number of free
parameters in the conditional distribution table.
This yields the cost function
\begin{equation}\label{eq:cost}
	cost = \sum_i c_i\, cost_i, \qquad
	cost_i = KL(p^D\|E(\theta_i)) + R(k_i,N).
\end{equation}
One concrete choice of $R(k_i,N)$ is the following
MDL-based penalty\footnote{The standard MDL criterion is based on the
	log-probability of the entire dataset and takes the form
	$(k\log N)/2$. Because $FC$ and $KL$ are defined as per-sample
	log-probabilities, we normalize accordingly.}
\citep{rissanen1989stochastic}:
\[
R(k_i,N) = \frac{k_i\log N}{2N},\quad k_i=|Y_i|(|X_i|-1).
\]
A key property of the cost function in Eq.~\eqref{eq:cost} is that it
decomposes into a sum of local cost functions $cost_i$, each depending
only on information local to node~$i$.
This decomposability allows the learning task to be solved as
independent local subproblems, one per node.

\subsection{Parameter Learning}
Parameter learning assigns a distribution over $X_i$ for each value
of $Y_i$.
More precisely, for a given information source $Y_i$, parameter
learning seeks the minimizer
\begin{equation}\label{eq:parameter_learning}
	\theta_i^{\min} = \arg\min_{\theta_i} KL(p^D\|E(\theta_i))
	+ R(k_i,N).
\end{equation}
Because $k_i$ and hence $R$ are constants once $Y_i$ is fixed, the
penalty can be omitted from the minimization, reducing
Eq.~\eqref{eq:parameter_learning} to
\[
\theta_i^{\min}=\arg\min_{\theta_i} KL(p^D\|E(\theta_i)).
\]
The following theorem characterizes the solution.

\begin{theorem}\label{thm:parameter_learning}
	Let $\theta_i^{\min}$ be the minimizer of
	$KL(p^D\|E(\theta_i))$ for a given information source $Y_i$.
	Then
	\begin{align}
		& \theta_i^{\min}(X_i\mid Y_i)
		=p^D(X_i\mid Y_i),\nonumber                 \\
		& KL(p^D\|E(\theta_i^{\min}))
		=H(p^D(X_i\mid Y_i))-H(p^D(X_i\mid X_{-i})).
		\label{eq:HtoKL}
	\end{align}
\end{theorem}
\begin{proof}
	See Appendix.
\end{proof}

\subsection{Structure Learning}\label{sec:structure_learning}
Given $Y_i$, parameter learning determines $\theta_i^{\min}$, and
the local cost function of node $i$ becomes
\begin{align}\label{eq:scost}
	cost_i(Y_i)
	& = KL(p^D\|E(\theta_i^{\min})) + R(k_i,N)\nonumber         \\
	& = H(p^D(X_i\mid Y_i))-H(p^D(X_i\mid X_{-i})) + R(k_i,N).
\end{align}
Structure learning seeks the minimizer
$Y_i^{\min}=\arg\min_{Y_i} cost_i(Y_i)$.
Because the second term of Eq.~\eqref{eq:scost} is constant with
respect to $Y_i$, it can be dropped, yielding the reduced cost
\begin{equation}\label{eq:cost_prime}
	cost'_i(Y_i)=H(p^D(X_i\mid Y_i))+R(k_i,N).
\end{equation}
This reduced cost coincides with the one used by
\citet{quinlan1989inferring}.
Its first term matches the cost function of the ID3 algorithm
\citep{quinlan1986induction}: the information source $Y_i$ should
reduce the conditional entropy of $X_i$ while being penalized for
complexity through $R$.
Moreover, evaluating $cost'_i(Y_i)$ is substantially cheaper than
evaluating $cost_i(Y_i)$, because the constant term
$H(p^D(X_i\mid X_{-i}))$ need not be computed.

To balance accuracy with computational cost, we minimize $cost'_i$ by
a greedy algorithm.
In what follows, we omit the node index $i$ since the algorithm
operates independently on each node.

\begin{algorithm}[t]
	\caption{Greedy structure learning}\label{alg:framework}
	$Y$: information source function\\
	$Cand$: set of candidates\\
	$c$: cost value
	\begin{algorithmic}[1]
		\STATE $Y \gets$ identically-zero function
		\LOOP
		\STATE $c \gets cost(Y)$
		\STATE Construct the candidate set $Cand$
		\STATE $Y^{\min} \gets \arg\min_{Y' \in Cand}\, cost(Y')$
		\IF{$cost(Y^{\min})\ge c$}
		\RETURN $Y$
		\ENDIF
		\STATE $Y \gets Y^{\min}$
		\ENDLOOP
	\end{algorithmic}
\end{algorithm}

Algorithm~\ref{alg:framework} presents a general greedy framework for
structure learning.
In each iteration, a \emph{candidate} is a function obtained by
modifying the current $Y$.
Various modification strategies can be considered; a concrete
instantiation is described in Section~\ref{sec:creating_candidate}.

\subsection{Convergence to the True Distribution}%
\label{sec:convergence_to_true_distribution}
A central question is whether the model distribution $\pi$ converges
to the true underlying distribution $p^*$ as the amount of training
data increases.
The results in this subsection apply to any learning algorithm that
follows the framework of Algorithm~\ref{alg:framework} and are not
restricted to a specific instantiation.

To formalize the convergence, we parameterize the relevant quantities
by $N$ (the number of training samples):
\begin{description}
	\item[$\omega$:] an infinite sequence of training samples
	\item[$\omega(N)$:] the first $N$ samples of $\omega$
	\item[$p^D(N)$:] empirical distribution of $\omega(N)$
	\item[$Y_i(N)$:] information source learned from $p^D(N)$
	\item[$\theta_i(N)$:] conditional distribution table learned from
	$p^D(N)$ and $Y_i(N)$
	\item[$k_i(N)$:] number of free parameters of $\theta_i(N)$
	\item[$\pi(N)$:] model distribution learned from $\omega(N)$
	\item[$Cand(N)$:] candidate set from which $Y_i(N)$ is selected.
\end{description}
This notation is used only in this subsection and in the
corresponding proofs in the Appendix.

We call $Y_i(N)$ a \emph{lossless information source} when it is a
one-to-one mapping.
Let $Y_i^{ll}$ denote a lossless information source for $X_i$.
If $Y_i(N)=Y_i^{ll}$, then
\begin{align}
	& H(p^D(X_i\mid Y_i(N)))
	= H(p^D(X_i\mid X_{-i})),\nonumber                                                     \\
	& KL(p^D\|E(\theta_i(N)))
	= 0\quad\text{(by Eq.~\eqref{eq:HtoKL})},\nonumber                                     \\
	& cost_i(Y_i(N))
	= R(k_i(N),N)\quad\text{(by Eq.~\eqref{eq:scost})}.\label{eq:cost(Y_i^ll)}
\end{align}

\begin{theorem}\label{thm:to_0}
	Let $cost_i$ be as in Eq.~\eqref{eq:cost}, that is,
	$cost_i(Y_i)=KL(p^D\|E(\theta_i))+R(k_i,N)$.
	Suppose the following conditions hold:
	\begin{enumerate}
		\item $Y_i^{ll}\in Cand(N)$ for all sufficiently large $N$.
		\item $R(k_i(N),N)\ge0$.
		\item $\lim_{N\to\infty}R(k_i^{ll},N)=0$, where
		$k_i^{ll}$ is the number of free parameters of
		$Y_i^{ll}$.
	\end{enumerate}
	Then
	\[
	\lim_{N\to\infty}KL(p^D(N)\|E(\theta_i(N)))=0.
	\]
\end{theorem}
\begin{proof}
	See Appendix.
\end{proof}

Theorem~\ref{thm:to_0} leads to the following convergence guarantee.

\begin{theorem}\label{thm:to_p_true}
	Let $cost_i$ be as in Eq.~\eqref{eq:cost}, that is,
	$cost_i(Y_i)=KL(p^D\|E(\theta_i))+R(k_i,N)$.
	Suppose the following conditions hold:
	\begin{enumerate}
		\item $Y_i^{ll}\in Cand(N)$ for all sufficiently large $N$.
		\item $R(k_i(N),N)\ge0$.
		\item $\lim_{N\to\infty}R(k_i^{ll},N)=0$, where
		$k_i^{ll}$ is the number of free parameters of
		$Y_i^{ll}$.
		\item The training data satisfy the strong law of large
		numbers, that is,
		\[
		\forall x\in X,\quad
		\lim_{N\to\infty}p^D(N)(x)=p^*(x)
		\quad\text{a.s.}
		\]
		\item $\forall x\in X, p^*(x)>0$
	\end{enumerate}
	Then
	\[
	\forall x\in X,\quad
	\lim_{N\to\infty}\pi(N)(x)=p^*(x)
	\quad\text{a.s.}
	\]
\end{theorem}
\begin{proof}
	See Appendix.
\end{proof}

\subsection{Creating Candidates}\label{sec:creating_candidate}
We now describe a concrete algorithm for constructing the candidate
set $Cand$ in Algorithm~\ref{alg:framework}.
The scheme is based on the decision graph construction algorithm of
\citet{chickering2013bayesian}, adapted to our cost function.
While \citet{chickering2013bayesian} describe their algorithm in terms
of decision graph construction, we present it as modifying the
information source $Y_i$.

Since we focus on a single node $i$ throughout this subsection, we
omit the subscript $i$ when the context is clear.

Define $L_y=\{x\in X\mid Y(x)=y\}$.\footnote{Because
	$Y(X)=Y(X_{-i})$, the collection $\{L_y\}$ is a partition of $X$
	that ignores the value of $X_i$.}
Following the convention in decision trees, we refer to each $L_y$ as
a \emph{leaf}.
The leaves $\{L_y\}_{y\in Y}$ form a partition of $X$: they are
pairwise disjoint and satisfy $\bigcup_{y\in Y}L_y=X$.
Conversely, given a set of leaves, the function $Y$ can be recovered
by assigning a distinct index to each leaf.
Thus, a function $Y$ and its corresponding set of leaves are
equivalent representations of the same partition.

We define two operations that produce a new partition from the current
one (see Figure~\ref{fig:split_merge}):

\begin{figure}[t]
	\centering\includegraphics{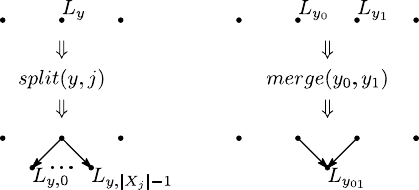}
	\caption{Split and merge operations on
		leaves.}\label{fig:split_merge}
\end{figure}

\begin{description}
	\item[$Split(y,j)$:] Split $L_y$ by the value of $x_j$
	($j\ne i$), creating new leaves
	\[L_{yx_j}=\{x\in L_y\mid X_j=x_j\},\]
	and remove $L_y$ from the leaf set.
	\item[$Merge(y_0,y_1)$:] Set
	$L_{y_{01}}=L_{y_0}\cup L_{y_1}$ and remove both
	$L_{y_0}$ and $L_{y_1}$ from the leaf set.
\end{description}

We define the cost of a leaf $L_y$ as
\begin{align*}
	Lcost(L_y)
	& = p^D(Y{=}y)\,H(p^D(X_i\mid Y_i{=}y))+R(k_{iy},N)\\
	& = -\frac{N_y}{N}\left(\sum_{x_i\in X_i}
	\frac{N_{yx_i}}{N_y}\log\frac{N_{yx_i}}{N_y}\right)
	+\frac{\log N}{2N}(|X_i|-1),
\end{align*}
where $N_y$ is the number of training samples in $L_y$ and
$N_{yx_i}$ is the number of training samples in $L_y$ with
$X_i=x_i$, and $k_{iy}=|X_i|-1$ is the number of free parameters in $L_y$.
The reduced cost $cost'$ then decomposes as a sum over leaves:
\begin{equation}\label{eq:sum_of_Lcost}
	cost'(Y)=\sum_{y\in Y}Lcost(L_y).
\end{equation}
The cost change $\Delta=cost'(Y^{\text{new}})-cost'(Y)$ induced by a
split or merge operation can be evaluated locally using
Eq.~\eqref{eq:sum_of_Lcost}:
\begin{equation}\label{eq:evaluating_Delta}
	\Delta=\begin{cases}
		\displaystyle\sum_{x_j\in X_j}
		Lcost(L_{yx_j})-Lcost(L_y)
		& \text{for }Split(y,j),     \\[6pt]
		Lcost(L_{y_{01}})-Lcost(L_{y_0})-Lcost(L_{y_1})
		& \text{for }Merge(y_0,y_1).
	\end{cases}
\end{equation}
Because evaluating Eq.~\eqref{eq:evaluating_Delta} requires only
information local to node $i$, the computational cost of learning is
significantly reduced.

Our algorithm extends decision tree algorithms such as ID3
\citep{quinlan1986induction}: whereas a standard decision tree uses
only split operations to construct a tree, our algorithm employs both
split and merge operations to construct a directed acyclic graph
(DAG).

When the DAG is updated by a split or merge operation, the new
function $Y^{\text{new}}$ is obtained by assigning indices to the
resulting leaves.
The assignment is arbitrary as long as distinct leaves receive
distinct indices; for concreteness, we adopt the following scheme.
\begin{description}
	\item[Case $Split(y,j)$:]
	\begin{equation}\label{eq:newYsplit}
		Y^{\text{new}}(x)=\begin{cases}
			Y(x)      & Y(x)<y,
			\\
			Y(x)-1    & Y(x)>y,
			\quad\text{(remove $L_y$ and shift indices left)}
			\\
			|Y|-1+x_j & Y(x)=y.
			\quad\text{(append split leaves to the end)}
		\end{cases}
	\end{equation}
	\item[Case $Merge(y_0,y_1)$, $y_0<y_1$:]
	\begin{equation}\label{eq:newYmerge}
		Y^{\text{new}}(x)=\begin{cases}
			Y(x)   & Y(x)<y_1,
			\\
			Y(x)-1 & Y(x)>y_1,
			\quad\text{(remove $L_{y_1}$ and shift indices left)}
			\\
			y_0    & Y(x)=y_1.
			\quad\text{(merge $L_{y_1}$ into $L_{y_0}$)}
		\end{cases}
	\end{equation}
\end{description}
This scheme ensures that $Y^{\text{new}}$ takes values in
$[0,|Y^{\text{new}}|)$, which is convenient for implementation.

During learning, the sequence of
operations---$Split(y,j)$ and $Merge(y_0,y_1)$---applied to produce
each $Y^{\text{new}}$ is recorded.
At inference time, given $x$, the value $Y(x)$ is initialized to $0$
and updated sequentially according to the recorded operations via
Eqs.~\eqref{eq:newYsplit} and~\eqref{eq:newYmerge}.
The complete table $\{Y(x)\}_{x\in X}$ need not be stored
explicitly, because each value $Y(x)$ can be reconstructed from the
recorded history.

\section{Inference by Conditional Pseudo-Gibbs Sampling}%
\label{sec:inference_by_pseudo-gibbs_sampling}
Suppose we wish to estimate $p^*(X_1X_2\mid X_0=x_0)$.
A straightforward approach is to construct the joint distribution
$\pi(X_0X_1X_2)$ via the learning algorithm and then compute
$\pi(X_1X_2\mid X_0=x_0)$ by Bayes' rule.
However, this approach becomes intractable when the number of
variables is large.

\citet{heckerman2000dependency} proposed the following alternative:
\begin{enumerate}
	\item Clamp $X_0$ to $x_0$.
	\item Run pseudo-Gibbs sampling without updating $X_0$.
	\item Treat the output sequence as samples drawn from
	$\pi'(X_1X_2\mid X_0=x_0)$, where $\pi'$ is a
	conditional distribution close to $\pi(X_1X_2\mid x_0)$.
\end{enumerate}

\begin{algorithm}[t]
	\caption{Conditional pseudo-Gibbs sampling}\label{algo:clamped}
	$X=\{X_i\}_{i=0}^{n-1}$: variables in the network\\
	$N^o$: number of samples to be drawn\\
	$\{X^t\}_{t=0}^{N^o-1}$: output sequence of the Markov chain\\
	$X_C$: variables clamped to $x_C$\\
	$X_{\bar{C}}$: unclamped variables
	\begin{algorithmic}[1]
		\STATE $X_C\gets x_C$
		\STATE $X_{\bar{C}}\gets x_{\bar{C}}$ \quad(arbitrary initial values)
		\FOR{$t = 0$ \TO $N^o-1$}
		\STATE $X^t\gets X$
		\STATE Select a node $i \in \bar{C}$ \label{line:select_not_clamped}
		\STATE Given inputs $Y_i(X_{-i})$, draw $x_i \sim \theta_i(X_i\mid Y_i)$
		\STATE $X_i\gets x_i$
		\ENDFOR
		\RETURN $\{X^t\}_{t=0}^{N^o-1}$
	\end{algorithmic}
\end{algorithm}

Algorithm~\ref{algo:clamped} presents the conditional pseudo-Gibbs
sampling procedure.
As in Algorithm~\ref{alg:pseudo}, the node $i$ at
line~\ref{line:select_not_clamped} can be selected by sequential scan
or by random scan with probability
\[
c_i=\begin{cases}
	\dfrac{1}{n-|C|} & i\in\bar{C} \\[4pt]
	0                & i\in C
\end{cases}
.
\]

Partitioning the index set $[0,n)$ into clamped nodes $C$ and
unclamped nodes $\bar{C}$, conditional pseudo-Gibbs sampling is
equivalent to pseudo-Gibbs sampling restricted to the nodes in
$\bar{C}$, where each node $i\in\bar{C}$ uses the conditional
distribution table
$\theta_i(X_i\mid X_{\bar{C}-\{i\}};\,X_C=x_C)$.
We can therefore decompose $KL(p\|E(\theta_i))$ as follows.
Let $B=\bar{C}-\{i\}$. Then
\begin{align}\label{eq:inference_KL}
	KL(p\|E(\theta_i))
	& = \left\langle\log\frac{p(X_i\mid X_{-i})}{\theta_i(X_i\mid X_{-i})}\right\rangle_p \nonumber            \\
	& = \left\langle\log\frac{p(X_i\mid X_B, X_C)}{\theta_i(X_i\mid X_B, X_C)}\right\rangle_{p(X_i, X_B, X_C)}
	\nonumber                                                                                                    \\
	& = \sum_{x_C}p(x_C)\left\langle\log\frac{p(X_i\mid X_B;\,X_C=x_C)}
	{\theta_i(X_i\mid X_B;\,X_C=x_C)}\right\rangle_{p(X_i,X_B\mid x_C)} \nonumber                               \\
	& = \sum_{x_C}p(x_C)\,KL(p(X_{\bar{C}}\mid x_C)\|E(\theta_i(X_i\mid X_B;\,x_C))).
\end{align}
Eq.~\eqref{eq:inference_KL} shows that $KL(p\|E(\theta_i))$ is a
weighted average of
$KL(p(X_{\bar{C}}\mid x_C)\|E(\theta_i(X_i\mid X_B;\,x_C)))$
over $x_C$.
Therefore, when $KL(p^*\|E(\theta_i))$ is small,
\[
KL(p^*(X_{\bar{C}}\mid x_C)\|E(\theta_i^{\min}(X_i\mid X_B;\,x_C)))
\]
is also small on average.
However, Eq.~\eqref{eq:inference_KL} also reveals that inference for
rare events---those with small $p^*(x_C)$---may become inaccurate,
because such terms receive little weight in the average.

\subsection{Lack of a Consistent Joint Distribution}
In a standard probabilistic model, training produces a single joint
distribution that is used for all subsequent inference queries.
By contrast, conditional pseudo-Gibbs sampling constructs a distinct
stationary distribution
$\pi'(X_{\bar{C}}\mid X_C=x_C)$ for each query $X_C=x_C$;
consequently, there is no single joint distribution $\pi(X)$ shared
across all inferences.

Inference by conditional pseudo-Gibbs sampling thus involves a
trade-off between computational tractability and the lack of a
consistent joint distribution.
This inconsistency is, however, immaterial for applications in which
only a single query is of interest.

\section{Related Work}\label{sec:related_work}
The idea of constructing a joint distribution from full conditional
distributions dates back to \citet{besag1974spatial}, who studied
Markov random fields and introduced the notion of pseudo-likelihood
in a subsequent paper \citep{besag1975statistical}.
\citet{geman1984stochastic} coined the term ``Gibbs sampling'' and
applied it to Markov random fields.
\citet{tresp1998nonlinear} proposed combining independently learned
conditional distributions via Gibbs sampling under the name
\emph{Markov blanket networks}.
\citet{heckerman2000dependency} formalized this line of work by
coining the term ``dependency network'' and introducing pseudo-Gibbs
sampling.
\citet{Lowd2011} empirically compared dependency networks with
Bayesian networks on real-world data.

On the theoretical side,
\citet{amari1985differential} introduced differential geometry into
statistics and laid the foundations of information geometry.
\citet{Takabatake12,takabatake2021reconsidering} developed an
information-geometric analysis of pseudo-Gibbs sampling, interpreting
each step as an m-projection onto a full conditional manifold, and
introduced the full conditional divergence; the present paper builds
on this work by deriving the FC-limit bound, reformulating structure
and parameter learning as cost-minimization problems, and proving the
convergence of the learned model distribution to the true underlying
distribution.

The cost function used in this paper coincides with the one proposed
by \citet{quinlan1989inferring}, who employed conditional entropy with
an MDL penalty for decision tree learning.

\section{Conclusion}\label{sec:conclusion}
A dependency network is a collection of regression mechanisms, each
estimating the conditional distribution of a single variable given all
others.
Pseudo-Gibbs sampling reconstructs a joint distribution from these
estimated full conditional distributions, and the learning task
decomposes into independent local subproblems for each node, keeping
the overall procedure tractable even as the number of nodes grows
large.

Pseudo-Gibbs sampling can be interpreted as iterative m-projections
onto the full conditional manifolds $\{E(\theta_i)\}$.
For an arbitrary distribution $p$ and the stationary distribution
$\pi$ of pseudo-Gibbs sampling, the inequality
\[
FC(p\|\pi)\le\sum_i c_i KL(p\|E(\theta_i))
\]
holds, where $FC$ denotes the full conditional divergence.
This inequality guarantees that if the right-hand side (the FC-limit)
is small, then $\pi$ is confined to a small neighborhood of $p$ in
the space of distributions.
Consequently, $\sum_i c_i KL(p\|E(\theta_i))$ serves as a principled
cost function for learning, and its minimization reduces to minimizing
conditional entropy.

We have proved that the model distribution converges to the true
underlying distribution as the number of training samples grows to
infinity, and have analyzed the accuracy of inference by conditional
pseudo-Gibbs sampling.

\section*{Acknowledgments}
This work was supported by JSPS KAKENHI (Grant No.\ 23K24909).

\appendix
\section*{Appendix}
\subsection*{Proof of Theorem \ref{thm:e-flat}}
Let $p_0,p_1$ be two distributions on the manifold $E(\theta_i)$, that is,
\begin{align*}
	p_0(X) & =p_0(X_{-i})\theta_i(X_i\mid X_{-i}) \\
	p_1(X) & =p_1(X_{-i})\theta_i(X_i\mid X_{-i})
\end{align*}
and let $p_\lambda$ be an arbitrary distribution on the e-geodesic $p_0p_1$ expressed by
\begin{align*}
	p_\lambda(X) & =Z^{-1}p_0(X)^{1-\lambda}p_1(X)^\lambda                                     \\
	& =Z^{-1}p_0(X_{-i})^{1-\lambda}p_1(X_{-i})^\lambda \theta_i(X_i\mid X_{-i}),
\end{align*}
where $Z$ is the normalizing constant.
Summing over $X_i$, we obtain
\[p_\lambda(X_{-i})=Z^{-1}p_0(X_{-i})^{1-\lambda}p_1(X_{-i})^\lambda\]
and
\[p_\lambda(X)=p_\lambda(X_{-i})\theta_i(X_i\mid X_{-i}).\]
Hence, $p_\lambda\in E(\theta_i)$.

Let $q_\lambda$ be an arbitrary distribution on the m-geodesic $p_0p_1$ expressed by
\begin{align*}
	q_\lambda(X)
	& =(1-\lambda)p_0(X)+\lambda p_1(X)                        \\
	& =\left((1-\lambda)p_0(X_{-i})+\lambda p_1(X_{-i})\right)
	\theta_i(X_i\mid X_{-i}).
\end{align*}
Summing over $X_i$, we obtain
\[q_\lambda(X_{-i})=(1-\lambda)p_0(X_{-i})+\lambda p_1(X_{-i})\]
and
\[q_\lambda(X)=q_\lambda(X_{-i})\theta_i(X_i\mid X_{-i}).\]
Hence, $q_\lambda\in E(\theta_i)$. \qed

\subsection*{Proof of Theorem~\ref{thm:m-projection}}
Let $q \in E(\theta_i)$. Then,
\begin{align*}
	KL(p\|q)
	 & = \left\langle\log\frac{p(X_{-i})}{q(X_{-i})}\right\rangle_{p(X_{-i})}
	+ \sum_{x_{-i}}p(x_{-i})\left\langle\log\frac{p(X_i\mid x_{-i})}{\theta_i(X_i\mid x_{-i})}
	\right\rangle_{p(X_i\mid x_{-i})},
\end{align*}
which is minimized when $q(X_{-i})\equiv p(X_{-i})$, that is, $q(X)=p(X_{-i})\,\theta_i(X_i\mid X_{-i})$.\qed

\subsection*{Derivation of Eq.~\eqref{eq:FC<=KL}}
\begin{align*}
	KL(p\|q)
	 & = \left\langle\log\frac{p(X)}{q(X)}\right\rangle_p                              \\
	 & = \left\langle\log\frac{p(X_i\mid X_{-i})}{q(X_i\mid X_{-i})}\right\rangle_p
	+ \left\langle\log\frac{p(X_{-i})}{q(X_{-i})}\right\rangle_p                       \\
	 & = \left\langle\log\frac{p(X_i\mid X_{-i})}{q(X_i\mid X_{-i})}\right\rangle_p
	+ \left\langle\log\frac{p(X_{-i})}{q(X_{-i})}\right\rangle_{p(X_{-i})}             \\
	 & = \left\langle\log\frac{p(X_i\mid X_{-i})}{q(X_i\mid X_{-i})}\right\rangle_p
	+ KL(p(X_{-i})\|q(X_{-i}))                                                         \\
	 & \ge \left\langle\log\frac{p(X_i\mid X_{-i})}{q(X_i\mid X_{-i})}\right\rangle_p.
\end{align*}
Multiplying both sides by $c_i$ and summing over $i$, we obtain
\begin{align*}
	KL(p\|q)
	 & \ge \sum_i c_i \left\langle\log\frac{p(X_i\mid X_{-i})}{q(X_i\mid X_{-i})}\right\rangle_p \\
	 & = FC(p\|q).
\end{align*}
\qed

\subsection*{Proof of Theorem \ref{thm:FC}}
Let $\pi_i$ denote the m-projection of $\pi$ onto $E(\theta_i)$, that is,
$\pi_i(X)=\pi(X_{-i})\theta_i(X_i\mid Y_i)$.
Then, the following inequality is obtained:
\begin{align*}
	\sum_i c_iKL(p\|E(\theta_i))-FC(p\|\pi)
	 & =\sum_i c_i\left\langle\log\frac{p(X_i\mid X_{-i})}{\theta_i(X_i\mid Y_i)}\right\rangle_p
	-\sum_i c_i\left\langle\log\frac{p(X_i\mid X_{-i})}{\pi(X_i\mid X_{-i})}\right\rangle_p                                       \\
	 & =\sum_i c_i\left\langle\log\frac{\pi(X_i\mid X_{-i})}{\theta_i(X_i\mid Y_i)}\right\rangle_p                                \\
	 & =\sum_i c_i\left\langle\log\frac{\pi(X_{-i})\pi(X_i\mid X_{-i})}{\pi(X_{-i})\theta_i(X_i\mid Y_i)}\right\rangle_p          \\
	 & =\sum_i c_i\left\langle\log\frac{\pi(X)}{\pi_i(X)}\right\rangle_p                                                          \\
	 & =\left\langle-\sum_i c_i\log\frac{\pi_i(X)}{\pi(X)}\right\rangle_p                                                         \\
	 & \ge\left\langle-\log\sum_i c_i\frac{\pi_i(X)}{\pi(X)}\right\rangle_p\quad
	\text{(since $-\log$ is convex)}                                                                                              \\
	 & =\left\langle-\log\frac{\pi(X)}{\pi(X)}\right\rangle_p\quad\text{(since $\sum_i c_i \pi_i=\pi$ by Eq.~\eqref{eq:p^(t+1)})} \\
	 & =0.
\end{align*}
\qed

\subsection*{Proof of Theorem~\ref{thm:parameter_learning}}
\begin{align*}
	KL(p\|E(\theta_i(X_i\mid Y_i)))
	 & = KL(p\|p\Pi_m(E(\theta_i)))                                                     \\
	 & = \left\langle\log\frac{p(X_i\mid X_{-i})}{\theta_i(X_i\mid Y_i)}\right\rangle_p \\
	 & = \left\langle\log\frac{p(X_i\mid Y_i)}{\theta_i(X_i\mid Y_i)}\right\rangle_p
	+ \left\langle\log\frac{p(X_i\mid X_{-i})}{p(X_i\mid Y_i)}\right\rangle_p           \\
	 & = \left\langle\log\frac{p(X_i\mid Y_i)}{\theta_i(X_i\mid Y_i)}\right\rangle_p
	+ H(p(X_i\mid Y_i)) - H(p(X_i\mid X_{-i}))                                          \\
	 & = \sum_{y_i}p(y_i)\,KL(p(X_i\mid y_i)\|\theta_i(X_i\mid y_i))
	+ H(p(X_i\mid Y_i)) - H(p(X_i\mid X_{-i})).
\end{align*}
Since the first term is non-negative and vanishes when $\theta_i(X_i\mid Y_i)\equiv p(X_i\mid Y_i)$, the minimum is attained at $\theta_i^{\min}=p(X_i\mid Y_i)$, giving
\[
	KL(p\|E(\theta_i^{\min})) = H(p(X_i\mid Y_i)) - H(p(X_i\mid X_{-i})).
\]
Substituting $p=p^D$ we complete the proof.
\qed

\subsection*{Proof of Theorem~\ref{thm:to_0}}
Assume that $N$ is sufficiently large.
Since $Y_i(N)=\arg\min_{Y_i\in Cand(N)} cost_i(Y_i)$,
and $Y_i^{ll}\in Cand(N)$,
\[0\le cost_i(Y_i(N))\le cost_i(Y_i^{ll}).\]
By Eq.~\eqref{eq:cost(Y_i^ll)}, $cost_i(Y_i^{ll})=R(k_i^{ll},N)$ and it converges to $0$ as $N\to\infty$.
Therefore,
\[\lim_{N\to\infty}cost_i(Y_i(N))=0,\] that is,
\[
	\lim_{N\to\infty}(KL(p^D(N)\|E(\theta_i(N)))+R(k_i(N),N))=0.
\]
Since $KL(p^D(N)\|E(\theta_i(N)))\ge0, R(k_i(N),N)\ge 0$, we obtain:
\[\lim_{N\to\infty}KL(p^D(N)\|E(\theta_i(N)))=0,\quad
	\lim_{N\to\infty}R(k_i(N),N)=0.\]
\qed

\subsection*{Proof of Theorem~\ref{thm:to_p_true}}
Let $\Omega$ be the set of sequences for which $p^D(N)$ converges to $p^*$, that is,
\[
\Omega = \{\omega\mid\forall x\in X, \lim_{N\to\infty}p^D(N)(x)=p^*(x)\},
\]
Here, assume $\omega\in\Omega$.
In the case $p^*(x)>0$, $p^D(N)(x)>0$ for all sufficiently large $N$.
In this condition, convergence in FC-divergence implies pointwise convergence.

By Theorem \ref{thm:FC},
\begin{equation}\label{eq:FC(N)<=KL(N)}
	FC(p^D(N)\|\pi(N))\le\sum_i c_i KL(p^D(N)\|E(\theta_i(N))).
\end{equation}
By Theorem~\ref{thm:to_0}, the right-hand side of Eq. \eqref{eq:FC(N)<=KL(N)} converges to $0$ as $N\to\infty$; therefore,
\[\lim_{N\to\infty}FC(p^D(N)\|\pi(N))=0.\]
Since 
$FC(p^D(N)\|\pi(N))=\sum_i c_i KL(p^D(N)\|\pi_i(N))$ and $c_i>0$,
\begin{align*}
	\lim_{N\to\infty}KL(p^D(N)\|\pi_i(N))
	&=\lim_{N\to\infty}\sum_x
	p^D(N)(x)\log\frac{p^D(N)(x)}{\pi_i(N)(x)},\quad
	\pi_i(N)=\pi(N)\Pi_m(E(\theta_i(N)))\\
	&=\sum_x p^*(x)\lim_{N\to\infty}\log\frac{p^*(x)}{\pi_i(N)(x)}\\
	&=\sum_x
	p^*(x)\lim_{N\to\infty}\log\frac{p^*(x_i\mid x_{-i})}
	{\pi_i(N)(x_i\mid x_{-i})}\quad\text{by Theorem \ref{thm:m-projection}}\\
	&=0
\end{align*}
for all $i$.
Since $p^*(x)>0$, $\pi_i(N)(x_i\mid x_{-i})$ converges to 
$p^*(x_i\mid x_{-i})$ as $N\to\infty$ for all $i$. 
Since every full conditional distribution of $\pi(N)$ 
coincides with that of $p^*$ in the limit,
\[\lim_{N\to\infty}\pi(N)(x)=p^*(x).\]
 
Since the strong law of large numbers gives $\Pr\{\omega\in\Omega\}=1$, we conclude
\[
\lim_{N\to\infty}\pi(N)(x) = p^*(x) \quad \text{a.s.}
\]
\qed

\vskip 0.2in
\bibliography{myref}
\end{document}